\documentclass[journal]{IEEEtran}
\usepackage{amsmath,amsfonts}
\usepackage{algorithmic}
\usepackage{array}
\usepackage[caption=false,font=normalsize,labelfont=sf,textfont=sf]{subfig}
\usepackage{textcomp}
\usepackage{stfloats}
\usepackage{url}
\usepackage{verbatim}
\usepackage{graphicx}
\usepackage{cite}
\usepackage{orcidlink}
\usepackage[ruled,vlined,linesnumbered]{algorithm2e}
	\SetKwInOut{Input}{input}
	\SetKwInOut{Output}{output}
\usepackage[utf8]{inputenc}
\usepackage[english]{babel}
\usepackage{booktabs}
\usepackage{csquotes}
\usepackage{makecell}
\usepackage{multirow}
\usepackage{threeparttable}
\usepackage{array}

\newcolumntype{C}[1]{>{\centering\arraybackslash}p{#1}}

\usepackage{caption}

\usepackage{amsthm}

\usepackage[svgnames,table]{xcolor}
\definecolor{darkblue}{rgb}{0, 0, 0.7}

\theoremstyle{plain}

\newtheorem{proposition}{Proposition}
\newtheorem{theorem}{Theorem}

\theoremstyle{definition}

\theoremstyle{remark}
\newtheorem{remark}{Remark}

\begin{document}

\title{General Performance Guarantee for Human Torque Estimation-Based Task-Agnostic Assistive Exoskeleton Control}

\author{
Duy Hoang \orcidlink{0009-0004-0030-1804},
Bastien Berret \orcidlink{0000-0002-0779-7724},
Olivier Bruneau \orcidlink{0000-0002-3485-2186}, and
Laurent Fribourg \orcidlink{0000-0002-5562-4078}
\thanks{This work has been submitted to the IEEE for possible publication. Copyright may be transferred without notice, after which this version may no longer be accessible.}
\thanks{This work is supported by a grant from the “Fondation CFM pour la Recherche", through the Jean-Pierre Aguilar fellowship.}
\thanks{Duy Hoang and Laurent Fribourg are with 
Universit\'e Paris-Saclay, CNRS, ENS Paris-Saclay,  LMF, 91190 Gif-sur-Yvette, France.
{(hoangduy@lmf.cnrs.fr; fribourg@lmf.cnrs.fr)}}
\thanks{Bastien Berret is with 
Universit\'e Paris-Saclay, Inria, CIAMS, 91190 Gif-sur-Yvette, France.
{(e-mail: bastien.berret@universite-paris-saclay.fr)}}
\thanks{Olivier Bruneau is with 
Universit\'e Paris-Saclay, ENS Paris-Saclay, LURPA, 91190 Gif-Sur-Yvette, France.
{(e-mail: olivier.bruneau@ens-paris-saclay.fr)}}
}

\markboth{Journal of \LaTeX\ Class Files,~Vol.~, No.~,}%
{Shell \MakeLowercase{\textit{et al.}}: A Sample Article Using IEEEtran.cls for IEEE Journals}


\maketitle

\begin{abstract}
Accurate human torque estimation is crucial for enabling task-agnostic control in robotic exoskeleton systems. However, estimation errors may cause mismatches between the robot assistance and the human intention, degrading controllability and task performance. In this paper, we address this issue by formally defining \textit{matched assistance} as scenarios in which the robot positively contributes to human movement. Based on this definition, we develop a theoretical framework to design the robot’s desired interaction torque that guarantees a lower bound on the matched assistance probability. Importantly, the proposed guarantee holds over the entire torque distribution, including unseen data beyond the training tasks. This provides our method with strong reliability and generalization, both of which are critical for effective exoskeleton control. The proposed strategy is implemented on the ABLE upper-limb exoskeleton and evaluated in a multi-task setup. Experimental results validate the theoretical guarantees and demonstrate that the proposed strategy achieves effective general performance across several tasks, guaranteeing movement smoothness while reducing human physical effort. 
\end{abstract}


\section{Introduction}
Assistive exoskeletons have attracted increasing attention as a means of providing physical assistance during human movement, with the control strategy playing a critical role in ensuring effective assistance while preserving natural human behavior \cite{proietti2016upper, wang2025control, guatibonza2024assistive}. Conventional exoskeleton controllers, however, are often designed for specific tasks, with assistance strategies tailored to predefined task requirements. For example, trajectory-tracking exoskeletons \cite{amiri2025fuzzy, sarani2026finite} rely on a predefined reference movement, point-reaching exoskeletons \cite{Orhan2025using, wang2019kinematic} require a target position, and load-carrying exoskeletons require either prior knowledge of the load weight \cite{wang2019multi} or an estimation of it \cite{liu2025using, nasiri2022human}. While such task-specific approaches can provide effective assistance under their well-defined conditions, their reliance on predefined task information may limit their applicability when users perform movements beyond these prescribed scenarios. For instance, a controller optimized for reaching a specific target may provide less effective assistance when applied to reaching tasks with different targets \cite{jamvsek2021predictive}, or the assistance designed for load-carrying tasks \cite{treussart2020controlling} may result in less smooth and intuitive assistance when applied to pick-and-place movements \cite{Quesada2025less}. These limitations have motivated the development of \textit{task-agnostic} control strategies \cite{hoang2025emg, Molinaro2024TaskAgnostic}, which aim to provide effective and intuitive assistance without requiring prior knowledge of the task being performed.

Task-agnostic exoskeleton control refers to the capability of the robot to achieve the desired behavior across a wide range of human activities, rather than being limited to a finite set of tested tasks or trajectories \cite{hoang2025emg, Molinaro2024TaskAgnostic}. To design such task-agnostic controllers, recent approaches rely on accurate estimation of human torque throughout the movement to determine the robot action \cite{Molinaro2024TaskAgnostic}. While the method proposed in \cite{Molinaro2024TaskAgnostic} relies on a large dataset encompassing multiple human tasks to train the torque estimation model, \cite{hoang2025emg} introduced an alternative approach that requires data from only a single training task while still achieving good generalization across the entire torque distribution. Building on the ideas of \cite{hoang2025emg}, our objective is to design a task-agnostic exoskeleton controller trained on a limited task dataset while maintaining reliable performance on previously unseen tasks.

To ensure effective generalization, the authors of \cite{hoang2025emg} formulated an upper bound on the generalization error over the entire data distribution based on Rademacher complexity theory \cite{XavierCF25, BartlettM02} and trained the torque estimation model to minimize this bound. Although \cite{hoang2025emg} achieves promising performance on unseen tasks by deriving a tight generalization error upper bound, errors remain unavoidable and probably lead to discrepancies between robot assistance and human intention. These discrepancies can result in two common forms of mismatch: (i) \textit{reversed assistance}, in which the robot actually opposes the human's motion, and (ii) \textit{excessive assistance}, in which a large torque produced by the exoskeleton causes the user to counteract the robot to maintain stability. Such interaction mismatches are undesirable since they can degrade user comfort, reduce control efficiency, and increase human effort, as observed in \cite{Quesada2025less, kang2019effect}. In contrast, \textit{matched assistance} describes scenarios in which the robot assists the movement in the correct direction without exceeding the level of assistance required for the task, allowing control to be shared with the human without resisting or taking over their actions. Accordingly, increasing the rate of matched assistance provides a potential solution for achieving compliant and smooth human-robot interaction (HRI). 

Motivated by these considerations, the present paper aims to achieve effective general control performance by ensuring a high probability of matched assistance for the user. Given a human torque estimation model with a known generalization error upper bound \cite{hoang2025emg}, we employ a \textit{dead-zone mechanism} \cite{treussart2020controlling, Quesada2025less} to generate the robot's desired torque $\tau_r$ from the estimated human torque $\hat{\tau}_h$. A dead-zone threshold $T_{dz}$ is utilized to separate the estimated human intention into low- and high-torque scenarios. Under the dead-zone, $\tau_r$ is set to zero in low-torque scenarios to let the human be fully in charge of the movement and increases in high-torque scenarios to allow the robot's assistance. The dead-zone threshold $T_{dz}$ is determined based on the generalization error upper bound to derive a lower bound on the matched assistance probability, thereby guaranteeing the robot’s general performance. The desired interaction torque $\tau_r$ is finally tracked using a low-level controller. Our contributions are summarized as follows:

\begin{enumerate}
    \item While human torque estimation has been employed for task-agnostic exoskeleton control \cite{Quesada2025less, hoang2025emg, Molinaro2024TaskAgnostic}, recent studies show limited attention to how estimation errors impact HRI and potentially lead to undesirable assistance. To address this gap, we introduce the concept of \textit{matched assistance}, referring to assistance that positively contributes to human movement, and use it as a criterion to characterize and mitigate the drawback of torque estimation errors.

    \item To obtain a high matched assistance probability, we propose a \textit{dead-zone mechanism} \cite{treussart2020controlling, Quesada2025less}. However, unlike \cite{treussart2020controlling, Quesada2025less}, which combine the dead-zone with an integral module for load-carrying tasks, we show that effective and generalizable assistance can be achieved using the dead-zone mechanism alone with an appropriately selected threshold $T_{dz}$. We also develop a theoretical approach to determine $T_{dz}$ that guarantees a lower bound on the matched assistance probability. This conservative approach ensures smoothness across a broad range of tasks while allowing the robot to contribute to the movement.  

    \item We implement and evaluate the proposed strategy on the ABLE exoskeleton \cite{Quesada2025less, hoang2025emg} in a \textit{multi-task} setting. Four test tasks are used to assess control generalization. Experimental results show that the proposed method preserves movement smoothness while reducing human physical effort by an average of $8.9\%$ compared to the transparent mode, across a wide range of tasks.
\end{enumerate}

The rest of this paper is organized as follows. Section \ref{sec: 2} introduces the proposed control strategy, in which the formal definition of matched assistance, the dead-zone mechanism, and the theoretical approach for selecting the dead-zone threshold are detailed. In Section \ref{sec: 3}, we describe the experimental setup and evaluation protocol. Experimental results are summarized in Section \ref{sec: 4}, and Section \ref{sec: 5} concludes our paper.

\section{Exoskeleton Control with General Performance Guarantee}
\label{sec: 2}

\subsection{Exoskeleton Control Strategy}

Our control strategy is developed from \cite{hoang2025emg} with three control levels as illustrated in Fig. \ref{fig: control strategy}. We use electromyography (EMG) signals as the input $\boldsymbol{x}$ of the torque estimation model since they reflect muscle activation associated with human movement exertion \cite{furukawa2021collaborative, liu2024human} and can provide information about motor intention before movement onset \cite{Quesada2025less, hoang2025emg}, thereby enabling torque estimation slightly before the actual motion. From EMG signals, the high-level controller incorporates an estimation model $f_x(\cdot)$ to predict human torque. The output $\hat{\tau}_h$ of the EMG-to-torque model is then modulated by a middle-level controller to derive the desired HRI torque $\tau_r$ that ensures compliant and smooth interaction between human and robot. The desired torque for the robot $\tau_r$ can thus be considered as a function of EMG signals and denoted as $\tau_r(\boldsymbol{x})$. At the low-level controller, the control signal $\tau_e$ is synthesized to regulate the HRI torque $\tau_{i}$ such that it tracks the desired reference torque produced by the mid-level layer. For this purpose, we employ a proportional–integral (PI) controller combined with a gravity compensator $\tau_{gc}$, as in \cite{hoang2025emg}.         

\begin{figure}[htp]
\centering
\includegraphics[scale=0.38]{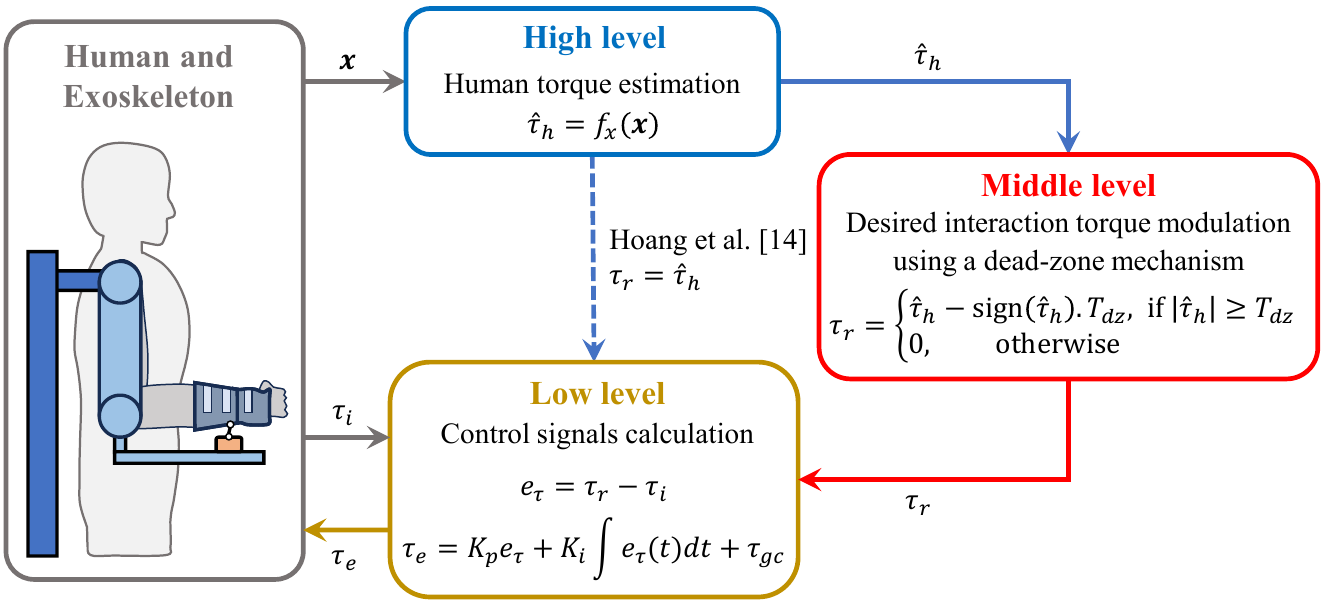}
\caption{Exoskeleton control strategy. Rather than directly passing the output of the high-level torque estimation model to the low-level controller as in~\cite{hoang2025emg}, we introduce a middle-level controller where the estimated torque is modulated by a dead-zone mechanism to ensure a high matched assistance probability before being applied to the low-level controller.}
\label{fig: control strategy}
\end{figure}

The main contribution of this work lies in the design of the middle-level controller. Given a high-level torque estimation model, our objective is to modulate its output into a safe desired interaction torque that can be delivered by the robot. To this end, Section \ref{subsec: matched assistance} introduces the concept of matched assistance, whereby the robot contributes positively to the user's intended movement. A higher probability of matched assistance is expected to promote smoother and more intuitive HRI. Building upon this concept, Section \ref{sec: modulating} presents a theoretical approach for designing the torque modulation strategy using a dead-zone mechanism to guarantee a high matched assistance probability over the entire data distribution. 

\subsection{Matched Assistance Probability}
\label{subsec: matched assistance}

To formally characterize the correlation between the real human torque~$\tau_h$ and the robot's desired torque $\tau_r(\boldsymbol{x})$ calculated from EMG signals $\boldsymbol{x}$, we consider a probability space $\left(\Omega, \mathcal{F}, P\right)$ in which $\Omega$ is the sample space containing all possible pairs of outcomes $(\boldsymbol{x},\tau_h)$, $\mathcal{F}$ is the event space of three assistance events: matched assistance~($M_a$), reversed assistance~($R_a$), and excessive assistance~($E_a$) defined as:
\begin{itemize}
    \item $M_a$: $\tau_r(\boldsymbol{x})\cdot\tau_h\geq0$ and $ |\tau_r(\boldsymbol{x})|\leq|\tau_h|$,
    \item $R_a$: $\tau_r(\boldsymbol{x})\cdot\tau_h<0$,
    \item $E_a$: $\tau_r(\boldsymbol{x})\cdot\tau_h\geq0$ and $ |\tau_r(\boldsymbol{x})|>|\tau_h|$,
\end{itemize}
and $P$ is the probability function. The \textit{matched assistance probability} $P(M_a)$ is defined as: 
\begin{equation}
\label{eq: def Pma}
    P(M_a) = P\left((\boldsymbol{x},\tau_h)\in M_a\right).
\end{equation}
In the context of general performance, the exoskeleton should interact effectively with the user while preserving controllability and natural movement execution across a broad range of tasks. Achieving such performance requires minimizing interaction mismatches that may interfere with the user's intended movement. Accordingly, we characterize general performance by a high matched assistance probability, which quantifies the likelihood that the robot's assistance is consistent with the human's intended movement. Having a guarantee on general performance is therefore equivalent to having a guarantee on the matched assistance probability $P(M_a)$.

We next show how the matched assistance probability can be used to interpret the HRI behavior produced by two task-agnostic controllers: the direct assistance~\cite{hoang2025emg} and the scaled assistance \cite{Molinaro2024TaskAgnostic}, as well as our proposed dead-zone-based assistance. The relation between the desired torque applied by each approach and the matched assistance probability is illustrated in Fig. \ref{fig: Assistance type}.

\begin{figure}[htp]
\centering
\includegraphics[scale=0.65]{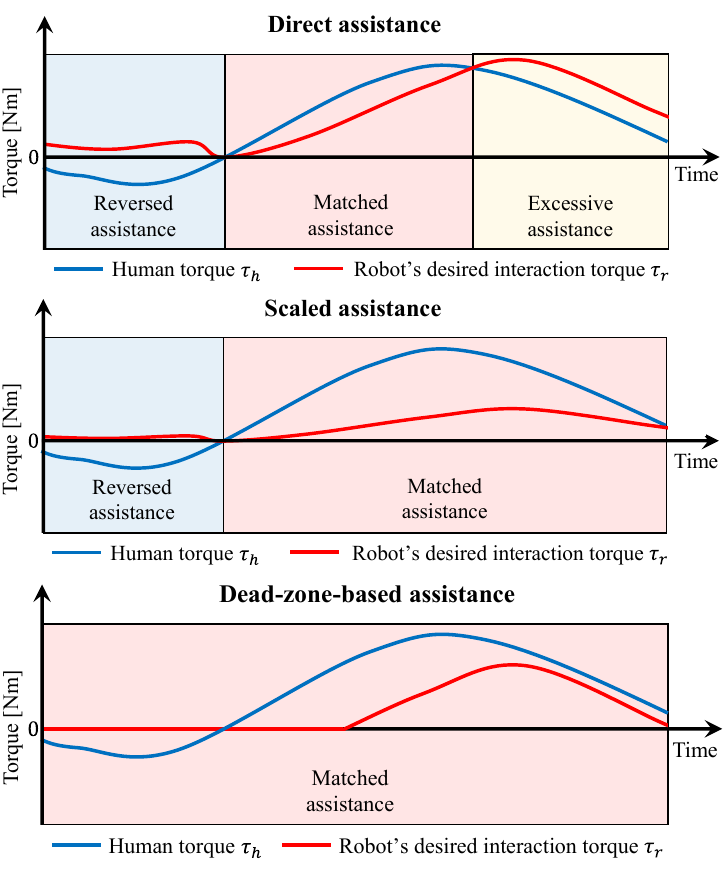}
\caption{Illustration of assistance behavior derived from the relationship between the human's intended and the robot's desired interaction torque under modulation approaches: direct assistance~\cite{hoang2025emg} (top plot), scaled assistance \cite{Molinaro2024TaskAgnostic} (middle plot), and the proposed dead-zone-based assistance (bottom plot).}
\label{fig: Assistance type}
\end{figure}

The direct assistance strategy proposed in \cite{hoang2025emg} directly uses the output of the torque estimation model as the desired HRI torque. Without any further modulation, its effectiveness therefore depends strongly on the accuracy of the high-level torque estimator. When estimation accuracy is insufficient, estimation errors can lead to reversed or excessive assistance (see the top plot of Fig. \ref{fig: Assistance type}), thereby degrading movement smoothness. To promote smoother HRI, Molinaro et al.~\cite{Molinaro2024TaskAgnostic} proposed a scaled-assistance strategy, in which the torque estimated by the high-level controller is scaled by a factor of $0.15$ to $0.2$ before being provided to the low-level controller. From the perspective of matched assistance, this strategy can be interpreted as reducing the magnitude of the torque applied under reversed mismatch. At the same time, scaling down the estimated torque reduces the likelihood of excessive assistance, thereby increasing the matched assistance probability (see the middle plot of Fig. \ref{fig: Assistance type}). However, this approach cannot remove reversed mismatches and limits the capacity of generating strong matched assistance. Another relevant condition is the transparent mode, in which $\tau_r=0$ theoretically, allowing the robot to follow the user without assistance. This condition corresponds to fully matched assistance and can preserve smooth and intuitive movement, as the user retains full control. Nevertheless, the lack of assistance requires greater human effort to perform a movement. Based on these observations, the objective of our proposed dead-zone-based assistance is to increase the matched assistance probability $P(M_a)$ while preserving a meaningful contribution of the robot during the task. With the dead-zone, we can reduce both reversed and excessive mismatches, as shown in the bottom plot of Fig. \ref{fig: Assistance type}. 

\subsection{Guarantee on Matched Assistance Probability}
\label{sec: modulating}

To increase the matched assistance probability $P(M_a)$ given by \eqref{eq: def Pma}, we design $\tau_r$ by modulating the estimated torque $\hat{\tau}_h$ from a high-level torque estimation model using a dead-zone. The dead-zone is defined with a positive threshold $T_{dz}$ as: 
\begin{equation}
\label{eq: dz}
    \tau_r = \begin{cases}
        \begin{aligned}
            &\hat{\tau}_h - \text{sign}(\hat{\tau}_h).T_{dz}, \quad\text{if $|\hat{\tau}_h|\geq T_{dz}$}\\
            &0, \quad \text{otherwise}
        \end{aligned}
    \end{cases}
\end{equation}

The main challenge of this approach is selecting the threshold $T_{dz}$, which should be sufficiently large to ensure a high matched assistance probability while remaining as small as possible to preserve assistance availability. In conventional studies \cite{treussart2020controlling}, \cite{Quesada2025less}, and \cite{Molinaro2024TaskAgnostic}, modulation parameters are tuned empirically by experts through experimental trials with the exoskeleton. This process requires conducting experiments and results in parameters whose performance is limited to the evaluated tasks, without guarantees for unexamined scenarios. In contrast, the proposed approach determines the dead-zone threshold $T_{dz}$ analytically without requiring task-specific experiments on the robot while ensuring general guarantees on control performance. Consider the data distribution $\mathcal{D}=\mathcal{X}\times\mathcal{T}$ of all possible EMG-torque pairs $(\boldsymbol{x},\tau_h)$ with $\boldsymbol{x}\in\mathcal{X}$ and $\tau_h\in\mathcal{T}$, we define the generalization error over~$\mathcal{D}$, denoted by $L_{\mathcal{D}}$ as:
\begin{equation}
\label{eq: Ld}
    L_{\mathcal{D}} = \mathbb{E}_{(\boldsymbol{x},{\tau}_h)\sim \mathcal{D}}\left(|{\tau}_h - f_x(\boldsymbol{x})|\right)
\end{equation} 
or briefly $L_{\mathcal{D}} =\mathbb{E}\left(|{\tau}_h - \hat{\tau}_h|\right)$. Our approach exploits the fact that $L_{\mathcal{D}}$
is upper bounded by $L_{\mathcal{B}^*}$ with high probability, as established in \cite{hoang2025emg}. The formulation of $L_{\mathcal{B}^*}$ is presented in Section \ref{subsec: training}. Accordingly, using $L_{\mathcal{B}^*}$, Section \ref{subsec: matching dz} delivers an analytical approach to design the dead-zone threshold $T_{dz}$ automatically.

\subsubsection{Guarantee on torque estimation generalization error}
\label{subsec: training}

Adapted from \cite{hoang2025emg, XavierCF25}, we present a framework to determine the upper bound $L_{\mathcal{B}^*}$ on the generalization error of the torque estimation model. The method applies to a class of estimation models with a linear output layer and consists of two main steps: (i) feature extraction and (ii) model retraining with bounded generalization error.

\paragraph{Feature extraction} 
Consider a torque estimation model $f_x(\cdot)$ with a linear output layer given by:
\begin{equation}
    f_x(\boldsymbol{x}) = \boldsymbol{a}^\top h(\boldsymbol{x}) =  \boldsymbol{a}^\top\boldsymbol{h}
\end{equation}
where $\boldsymbol{x}\in \mathbb{R}^d$ is the input vector, $h(\boldsymbol{x})$ represents the feature extractor from the input $\boldsymbol{x}$ that contains all layers before the last linear output layer of model $f_x(\cdot)$, the feature extractor produces a feature vector $\boldsymbol{h} \in \mathbb{R}^m$, which is finally converted into the human torque via a linear output layer with the weight $\boldsymbol{a} \in \mathbb{R}^m$. In this step, the model $f_x(\cdot)$ is trained using any conventional training strategy to obtain a feature extractor $h(\cdot)$ from the training dataset, without considering constraints on general performance.

\paragraph{Model retraining with bounded generalization error} We freeze the feature extractor derived from the previous step and retrain the last linear output layer to obtain the tightest upper bound on the generalization error. Since the feature extractor is fixed, the training of $f_x(\cdot)$ is equivalent to fitting a linear model $f_h(\boldsymbol{h}) = \boldsymbol{a}^\top\boldsymbol{h}$ with input $\boldsymbol{h}$ from the feature space $\mathcal{H}$ to the corresponding output torque $\tau_h$ from~$\mathcal{T}$. The training of $f_h(\boldsymbol{h})$ is performed to minimize the generalization error $L_{\mathcal{D}}$ defined by (\ref{eq: Ld}). Since~$\mathcal {D}$ is not accessible, the estimation model is trained on a training subset $\mathcal{S} = \mathcal{X}_{\mathcal{S}} \times \mathcal{T}_{\mathcal{S}}$ of $n$ samples drawn independently and identically distributed from $\mathcal{D}$ with $\mathcal{X}_{\mathcal{S}} \subset \mathcal{X}$ and $\mathcal{T}_{\mathcal{S}} \subset \mathcal{T}$.

Similar to \cite{hoang2025emg, XavierCF25}, we use gradient descent (GD) to train the linear model $f_h(\cdot)$ to minimize the quadratic loss function ${\cal L}(\boldsymbol{W})=\frac{1}{2}\sum_{i=1}^n|v_i|^2$ in which the training error $v_{i}(k) = \boldsymbol{a}(k)^\top h(\boldsymbol{x}_i) - \tau_{h, i}$ represents the estimation error of the $i^{th}$ sample. At each step $k$ of GD, with a probability at least $1-\delta$, the upper bound $L_{\mathcal{B}}(k)$ on $L_{\mathcal{D}}$ is formulated as:
\begin{equation}
\label{eq: LB}
    L_{\mathcal{B}}(k) = \frac{1}{n}\sum_{i=1}^n|v_{i}(k)|+ \frac{2\mathcal{C}}{\sqrt{n}}\|\boldsymbol{a}(k)\| + 3R\sqrt{\frac{\log \frac{2}{\delta}}{2n}}
\end{equation}
where $\boldsymbol{a}(k)$ is the output weight at $k^{th}$ GD step, the parameters $\mathcal{C} = \max_{i\in[n],\;\boldsymbol{x}_i \in \mathcal{X}_S}\|h(\boldsymbol{x_i})\|$ and $R\geq |f_h(h(\boldsymbol{x})) - \tau_h|$ $\forall \;(\boldsymbol{x},\tau_h)\sim {\mathcal{D}}$. Following the early stopping strategy \cite{XavierCF25}, the training process is subsequently halted at the step $k^*$ when $L_{\mathcal{B}}$ starts increasing to obtain a tight upper bound $L_{\mathcal{B}^*} =L_{\mathcal{B}}(k^*)$. From \cite{ma22}, we have:
\begin{proposition}\label{theo: Rademacher}
With a probability at least $1-\delta$ over the sample~$\mathcal{S}$ of size $n$, 
the generalization error $L_{\mathcal{D}}$ satisfies: 
\begin{equation}\label{eq: lb*}
L_{{\mathcal{D}}}(f_h) 
\leq  L_{\mathcal{B}^*}. 
\end{equation}

\end{proposition}
\begin{proof}
    See Appendix A.
\end{proof}

\subsubsection{Analytical approach for designing the dead-zone}
\label{subsec: matching dz}
With the generalization error upper bound $L_{\mathcal{B}^*}$ obtained from Section \ref{subsec: training}, we now present our theoretical approach to set the threshold $T_{dz}$. We first consider the relation between the matched assistance probability $P(M_a)$ defined by \eqref{eq: def Pma} and the dead-zone threshold $T_{dz}$, given by the following proposition: 

\begin{proposition}
\label{prop: relation PMa}
For a given dead-zone threshold~$T_{dz}$, the matched assistance probability $P(M_a)$ satisfies:
\begin{equation}
\label{eq: app 1}
    P(M_a)\geq P\left(|\hat{\tau}_h-\tau_h|\leq T_{dz}\right).
\end{equation}
\end{proposition}
\begin{proof}
We divide the human torque $\tau_h$ from the entire data distribution into 2 cases: $\tau_h <0$ and $\tau_h \geq 0$. Considering the case $\tau_h <0$, the dead-zone provides a matched $\tau_h \leq \tau_r \leq 0$ if $\tau_h - T_{dz}\leq \hat{\tau}_h \leq T_{dz}$. Suppose that $\tau_h <0$ is already held, the matched probability in this case satisfies:
\begin{equation}
\label{eq: Pma bound sigma}
\begin{aligned}
    P(M_a) =&\; P(\tau_h - T_{dz}\leq \hat{\tau}_h \leq T_{dz}) \\
    =&\; P(\hat{\tau}_h \leq T_{dz}) - P(\tau_h - T_{dz}\leq \hat{\tau}_h)\\
    =&\; P(\hat{\tau}_h - \tau_h\leq T_{dz} - \tau_h) - P(\tau_h - \hat{\tau}_h\leq T_{dz})\\
    \geq&\; P(\hat{\tau}_h - \tau_h\leq T_{dz}) - P(\hat{\tau}_h - \tau_h \geq -T_{dz}) \\
    =&\;P\left(|\hat{\tau}_h-\tau_h|\leq T_{dz}\right).
\end{aligned}
\end{equation}
By applying a similar process to the case   $\tau_h\geq0$, we also have $P(M_a)\geq P\left(|\hat{\tau}_h-\tau_h|\leq T_{dz}\right)$. Since these two cases cover the entire torque distribution, we derive (\ref{eq: app 1}).
\end{proof}

Proposition \ref{prop: relation PMa} indicates that the matched probability $P(M_a)$ does not only depend on $T_{dz}$ but also on the distribution of the estimation error $(\tau_h - \hat{\tau}_h)$. While the exact distribution of the estimation error is generally unknown, empirical observations in Fig. \ref{fig: Error distribution} suggest that, for a high-quality torque estimation model, the error can be reasonably characterized by a zero-mean normal distribution. Accordingly, we assume:
\begin{equation}
\label{eq: assum}
        \left( {\tau}_h - \hat{\tau}_h\right)\sim\mathcal{N}\left(0,\sigma^2\right).
\end{equation}
From this assumption, the selection of $T_{dz}$ is undertaken based on the following theorem:

\begin{theorem}\label{theo: dz} If the estimation error of the estimation model follows a normal distribution with zero mean $\mathcal{N}(0,\sigma^2)$ and generalization error $L_{\mathcal{D}} = \mathbb{E}\left(|{\tau}_h - \hat{\tau}_h|\right)$ satisfies $L_{\mathcal{D}}\leq L_{\mathcal{B}^*}$, then, by using the dead-zone mechanism with the threshold $T_{dz}$, the matched assistance probability $P(M_a)$ is guaranteed to satisfy:
\begin{equation}
\label{eq: pb}
P\left(M_a\right) 
\geq  2\Phi\left(\sqrt{\frac{2}{\pi}}\frac{T_{dz}}{L_{\mathcal{B}^*}}\right) - 1
\end{equation}
where $\Phi\left(\cdot\right)$ is the standard normal cumulative distribution function (CDF).
\end{theorem}
\begin{proof}
From the distribution $ \left( {\tau}_h - \hat{\tau}_h\right)\sim\mathcal{N}\left(0,\sigma^2\right)$, we have:
\begin{equation}
    L_{\mathcal{D}} = \mathbb{E}\left(|{\tau}_h - \hat{\tau}_h|\right) = \sqrt{\frac{2}{\pi}}\sigma.
\end{equation}
Accordingly, since $L_{\mathcal{D}}\leq L_{\mathcal{B}^*}$, we derive:
\begin{equation}
\label{eq: sigma}
    \sigma = \sqrt{\frac{\pi}{2}}L_{\mathcal{D}}\leq\sqrt{\frac{\pi}{2}}L_{\mathcal{B}^*}.
\end{equation}
Next, from Proposition \ref{prop: relation PMa} we have: 
\begin{equation}
\label{eq: PMa by sigma}
\begin{aligned}
    P(M_a) =&\;P\left(|\hat{\tau}_h-\tau_h|\leq T_{dz}\right)\\
    =&\;P\left(\tau_h - \hat{\tau}_h \leq T_{dz}\right) - P\left(\tau_h - \hat{\tau}_h \geq -T_{dz}\right)\\
    =&\;\Phi\left(\frac{T_{dz}}{\sigma}\right) - \Phi\left(\frac{-T_{dz}}{\sigma}\right)=\;2\Phi\left(\frac{T_{dz}}{\sigma}\right) - 1.
\end{aligned}
\end{equation}
Substituting \eqref{eq: sigma} into \eqref{eq: PMa by sigma}, note that the CDF function $\Phi(\cdot)$ is non-decreasing, we obtain \eqref{eq: pb}.
\end{proof}

\begin{remark}
\label{re: set dz}
    From Theorem \ref{theo: dz}, by formulating an upper bound $L_{\mathcal{B}^*}$ on the generalization error, we can guarantee a lower bound on the matched assistance probability $P(M_a)$ for any given dead-zone threshold $T_{dz}$, and \textit{this is obtained without explicitly knowing} $\sigma$. Theorem \ref{theo: dz} thus provides a practical guideline for selecting the dead-zone threshold $T_{dz}$. 
       
    In practice, the dead-zone threshold introduces a trade-off between movement smoothness and the level of robotic assistance. For intuitive and safe HRI, we prioritize movement smoothness over maximizing assistance. The threshold $T_{dz}$ should be designed as small as possible to avoid a lack of assistance but large enough to suppress mismatched interaction and thereby improve human controllability and movement smoothness. Accordingly, we suggest setting $T_{dz}=2L_{\mathcal{B}^*}$, which provides a guaranteed matched assistance probability of $P(M_a)$ of at least:
\begin{equation}
\label{eq: Pb*}
    P_{\mathcal{B}^*} = 2\Phi\left(2\sqrt{\frac{2}{\pi}}\right) - 1 = 0.89.
\end{equation}
\end{remark}

\begin{remark}
Although we assume that the estimation error follows a zero-mean normal distribution $\mathcal{N}(0,\sigma^2)$, which is in accordance with our empirical observations in Fig. \ref{fig: Error distribution}, a less restrictive form of Theorem \ref{theo: dz} can be derived without imposing the zero-mean assumption. This is illustrated in Appendix B.
\end{remark}

\section{Experimental Method}
\label{sec: 3}
We consider two experiments to validate: (i) the guarantee of the proposed method on the matched assistance probability $P(M_a)$, and (ii) the performance of the control strategy in practical exoskeleton control under a task-agnostic setup. 

\subsection{Experiment on General Performance Guarantee} This experiment aims to examine the impact of using the dead-zone in increasing matched assistance probability and validate the theoretical choice of the dead-zone threshold to guarantee a lower bound on the matched assistance probability given by Theorem \ref{theo: dz} and Remark \ref{re: set dz}. 

\subsubsection{Participant and task} We conducted this experiment using the 17-subject multi-joint dataset provided in \cite{quesada_dataset}. The data collection process and the calculation of the human torque label $\tau_h$ are described in \cite{Quesada2024} and summarized in Section~\ref{subsec: data collection}. Participants were connected to the exoskeleton at the wrist level via an orthosis mounted at the end of the forearm of the exoskeleton, as in \cite{verdel2022influence}. Each participant performed 10 trials of random trajectory tracking while wearing an upper-limb exoskeleton. All movements were executed in the parasagittal plane and involved only shoulder and elbow flexion/extension. 

\subsubsection{Torque estimation models} We applied the bounded-generalization-error framework (Section \ref{subsec: training}) to five EMG-to-torque architectures, including: the multivariate linear regression (MVLR) model \cite{camardella2021}, which directly maps filtered EMG signals to human torque without feature extraction; the nonlinear mapping (NLMap) model \cite{Quesada2025less}, which uses a nonlinear activation for feature extraction; the feedforward neural network \cite{zhang2020ankle} with four hidden layers (4HNN) model, which utilizes four hidden layers as the feature extractor; the LSTM model \cite{zhang2022lower}, which employs a sequence of LSTM block followed by fully connected layers as the feature extractor; and the CNN-LSTM model \cite{hoang2025from}, which uses a hybrid CNN-LSTM architecture to extract spatio-temporal information from filtered EMG signals followed by fully connected layers as the feature extractor.

\subsubsection{Evaluation of generalization guarantee} For each subject, the 10 trials were split into a training set of 3 trials and a test set of 7 trials. A larger test set was intentionally adopted since the generalization error cannot be directly accessed; therefore, a large test set enables estimation of these quantities, as suggested in \cite{XavierCF25, zhang2017understanding}. Consider the test set of $n_{test}$ EMG-torque samples $\{(\boldsymbol{x}_j, \tau_{h,j})\}_{j=1}^{n_{test}}$, for the $j^{th}$ test sample, we denote the corresponding estimation given by the torque estimation model as $\hat{\tau}_{h, j}$ and the corresponding robot torque obtained by \eqref{eq: dz} as ${\tau}_{r, j}$. The generalization error is estimated by the test error $L_{test}$ defined as the mean absolute error:
\begin{equation}
    L_{test} = \frac{1}{n_{test}}\sum_{j=1}^{n_{test}}|\tau_{h,j} - \hat{\tau}_{h,j}|
\end{equation}
We first verify if $L_{test} \leq L_{\mathcal{B}^*}$ as indicated by Proposition~\ref{theo: Rademacher}. Then, using the generalization error upper bound $L_{\mathcal{B}^*}$, we set the dead-zone $T_{dz} = 2L_{\mathcal{B}^*}$ and calculate the matched assistance probability $P(M_a)$ of each model as:
\begin{equation}
\label{eq: calculate Pma}
    P(M_a) = \frac{n_{M_a}}{n_{test}},
\end{equation}
in which $n_{M_a}$ is the number of matched assistance time sampled by comparing $\tau_{r,j}$ to the corresponding $\tau_{h,j}$ in the test set. From \eqref{eq: calculate Pma} and \eqref{eq: Pb*}, we further verify if $P(M_a) \geq 0.89$, as indicated by  Remark \ref{re: set dz} and Theorem \ref{theo: dz}.   

\subsubsection{Matched assistance probability versus assistance} 
To quantify the contribution of the desired HRI torque $\tau_r(\boldsymbol{x})$ relative to the human torque $\tau_h$, we introduce the \textit{positive contribution index} (PCI). Inspired by the assistance index~\cite{fanti2022assessment}, the PCI is adapted to the proposed concept of matched assistance. For each pair of $\tau_r(\boldsymbol{x})$ and $\tau_h$, the PCI takes a positive value when the robot assistance is matched with the human's intended movement and a negative value when the assistance is mismatched. The PCI is defined as follows:

\begin{equation}
\label{eq: PCI}
    \text{PCI} = \begin{cases}
        \begin{aligned}
            &|\tau_r(\boldsymbol{x})|, \quad\text{if $(\boldsymbol{x},\tau_{h})\in M_a$}\\
            &- |\tau_r(\boldsymbol{x})|, \quad\text{if $(\boldsymbol{x},\tau_{h})\in R_a$}\\
            & - |\tau_r(\boldsymbol{x}) - \tau_h|\quad\text{if $(\boldsymbol{x},\tau_{h})\in E_a$}
        \end{aligned}
    \end{cases}
\end{equation}

Since mismatched assistance can adversely affect human controllability, $P(M_a)$ is treated as the primary performance criterion for the controller design. On the other hand, PCI provides complementary information regarding the magnitude of the robot's contribution but does not replace the requirement for reliable matched assistance. Thus, $P(M_a)$ is prioritized, and PCI is considered only when the corresponding $P(M_a)$ values are comparable. Among such configurations, the one yielding a higher PCI is preferred.

\subsection{Experiment on EMG Assistance}
This experiment aims to validate the generalizability and effectiveness of the proposed control strategy in a multi-task setup. The controller training and design were based on a single standard trajectory tracking task, whereas four additional tasks were used for testing.

\subsubsection {Participant and task} The assistance experiment was conducted with 5 healthy subjects (4 males, average height of 1.72m, average weight of 70kg) performing a series of motor tasks in the parasagittal plane involving shoulder and elbow flexion/extension. Participants had no prior experience with the robot. Written informed consent was provided, and the protocol was approved by the ethics committee of Université Paris-Saclay. Each participant completed two sessions.

In the first session, data were collected to train the high-level torque estimation model. Following \cite{hoang2025emg, Quesada2024}, we consider the random trajectory tracking task as the training task. The participant tracked a randomly generated trajectory projected on a screen. Each trajectory lasted 30 seconds, and~3 trajectories were performed in this step. During the training task, the robot was set to viscous mode \cite{Quesada2024} to increase the effort required from the participant. This allows capturing a large range of muscle activation. 

\begin{figure}[htp]
\centering
\includegraphics[scale=0.38]{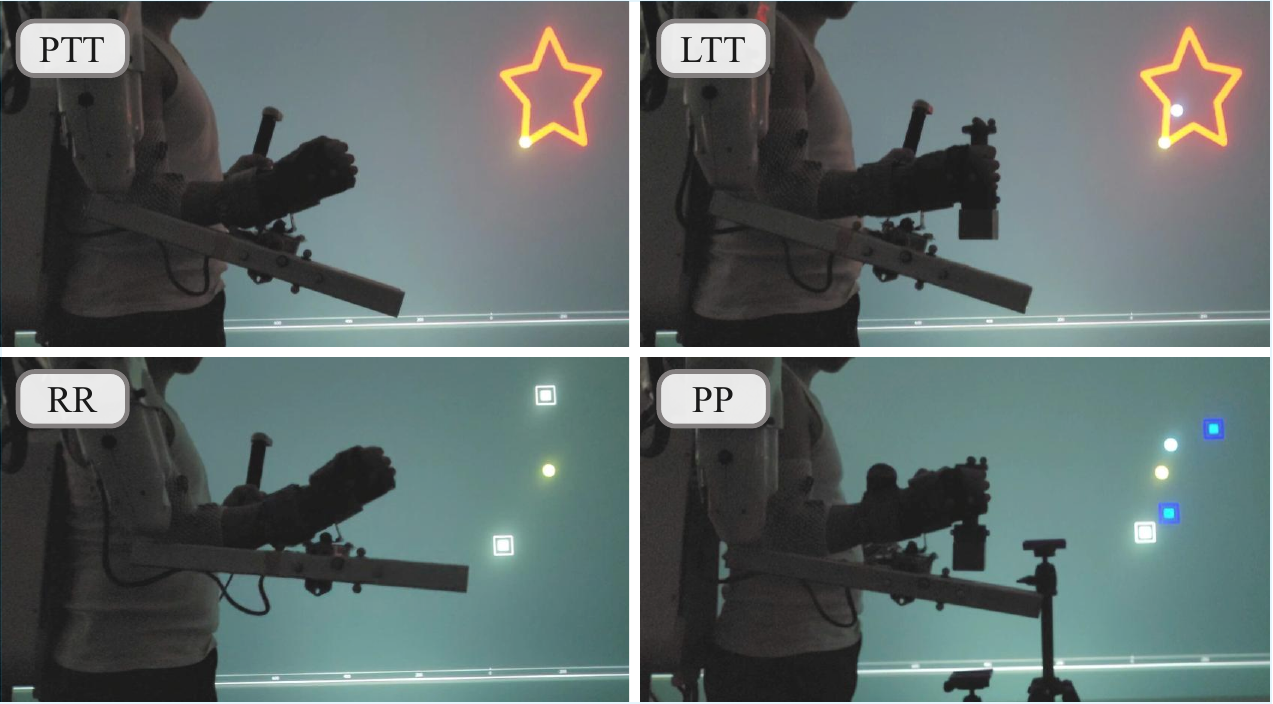}
\caption{Testing tasks: PTT = Pure trajectory tracking, LTT = Load trajectory tracking, RR = Reach and return, PP = Pick and place.}
\label{fig: testing tasks}
\end{figure}

In the second session, the participant performed four testing tasks (see Fig. \ref{fig: testing tasks}) under three control conditions: the proposed dead-zone-based assistance (DZA) with $\tau_r$ calculated by (\ref{eq: dz}), the direct assistance (DA) that directly use $\tau_r = \hat{\tau}_h$ (or with $T_{dz} = 0$) as in \cite{hoang2025emg}, and the transparent (TR) mode in which $\tau_r = 0$. The TR mode is used as the baseline for evaluating human controllability and movement smoothness, since it allows the human to remain fully in charge of task execution. We consider four testing tasks:
\begin{itemize}
    \item Pure trajectory tracking (PTT): The participant tracks a star trajectory, repeated 3 times.
    \item Load trajectory tracking (LTT): Similar to PTT, but the participant holds a load of 1.5 kg during the task.
    \item Reach and return (RR): The participant moves to a point at a high position, then returns to a point at a low position, repeated 5 times.
    \item Pick and place (PP): The participant picks a 1.5 kg load from the higher platform, moves it to the lower platform, then picks the load again and returns it to the higher platform, repeated 5 times.
\end{itemize}
The testing tasks and control conditions were organized in a random order.

\subsubsection{Data collection and processing}
\label{subsec: data collection}
We employ the methodology of \cite{Quesada2024} for the processing of EMG signals, the recording of human motion, and the calculation of the human torque. The EMG signals are collected using MiniWave sensors (Cometa, Bareggio MI, Italy) on eight muscles: the brachioradialis, brachialis, the long, medial, and lateral heads of the triceps, and the anterior, posterior, and medial deltoids. Raw EMG data were filtered by a fourth-order Butterworth bandpass filter (20–450~Hz), then centered, rectified, processed to extract the envelope using a 3~Hz low-pass filter, and finally normalized by the maximum voluntary contraction (MVC) of the corresponding muscle. Human motion is recorded using a ten-camera motion capture system. The human–robot interaction force is measured using a force/torque sensor mounted at the robot’s end effector. The motion and interaction data are then sent to the OpenSim Inverse Dynamics tool to calculate the human joint torque labels $\tau_h$ exerted during the task.

\subsubsection{Assistive controller design} The assistive controller was designed to support both shoulder and elbow flexion/extension. The two-output 4HNN model was employed as the high-level controller to estimate the shoulder and elbow torques. Each hidden layer of the 4HNN model contains 32 neurons with the $tanh(\cdot)$ activation function. The 4HNN model was selected for the experiments because it provides a tighter upper bound on the generalization error than the MVLR and NLMap models for both shoulder and elbow torque estimation, as shown in Fig. \ref{fig: guarantee}A. Although the 4HNN yields a less favorable $L_{\mathcal{B}^*}$ than the deeper LSTM and CNN-LSTM models, it requires substantially less training time, taking less than 5 min compared with more than 45 min for the deep models. This trade-off between generalization performance and computational cost makes the 4HNN a more practical choice for the experimental setup, particularly by reducing the waiting time required for participant-specific model training. 

For each subject, we trained the 4HNN model on the corresponding dataset collected from the random trajectory tracking training task. After obtaining the estimation model, a dead-zone with threshold $T_{dz} = 2L_{\mathcal{B}^*}$ was applied at the middle-level controller. Across 5 testing subjects, average dead-zone thresholds of $8.12$ N$\cdot$m and $5.76$ N$\cdot$m were used for shoulder and elbow torque, respectively. At the low-level controller, we set $K_P=0.3$ and $K_I=1$, as used in \cite{Quesada2025less}.

\subsubsection{Evaluation of assistance performance} We compare the proposed DZA control strategy with the DA controller of~\cite{hoang2025emg} and the TR mode across the four testing tasks in terms of movement smoothness and human physical effort. On the one hand, movement smoothness is quantified as the integral of the squared jerk (ISJ) over the movement duration. On the other hand, human physical effort is evaluated as the average activation of the eight muscles during task execution. To eliminate task-dependent bias, the activation of each muscle was first normalized by its task-specific maximum voluntary contraction (task MVC), defined as the maximum activation recorded during the task under the TR condition, before averaging across muscles.

\begin{figure*}[htp]

\centering
\makebox[0.48\textwidth][l]{\fontsize{10pt}{10pt}\selectfont \textbf{A}}
\makebox[0.48\textwidth][l]{\fontsize{10pt}{10pt}\selectfont \textbf{B}}
\includegraphics[scale=0.42]{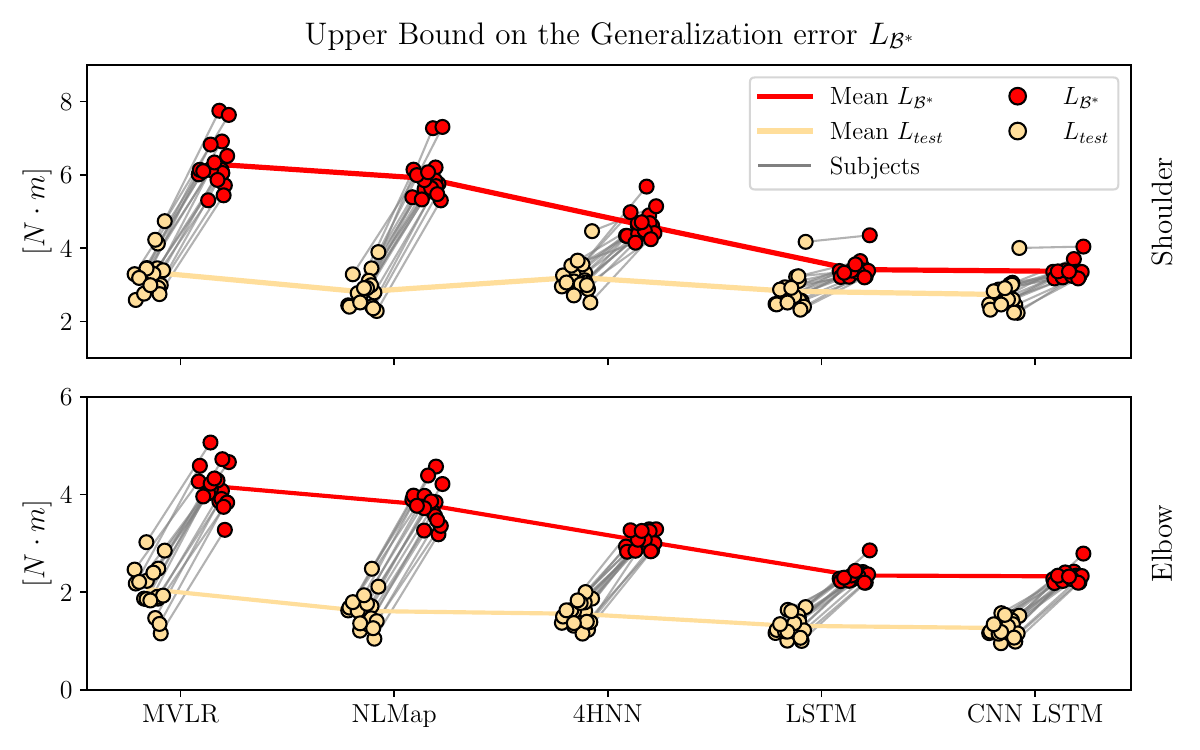}\hspace{0.5cm}\includegraphics[scale=0.42]{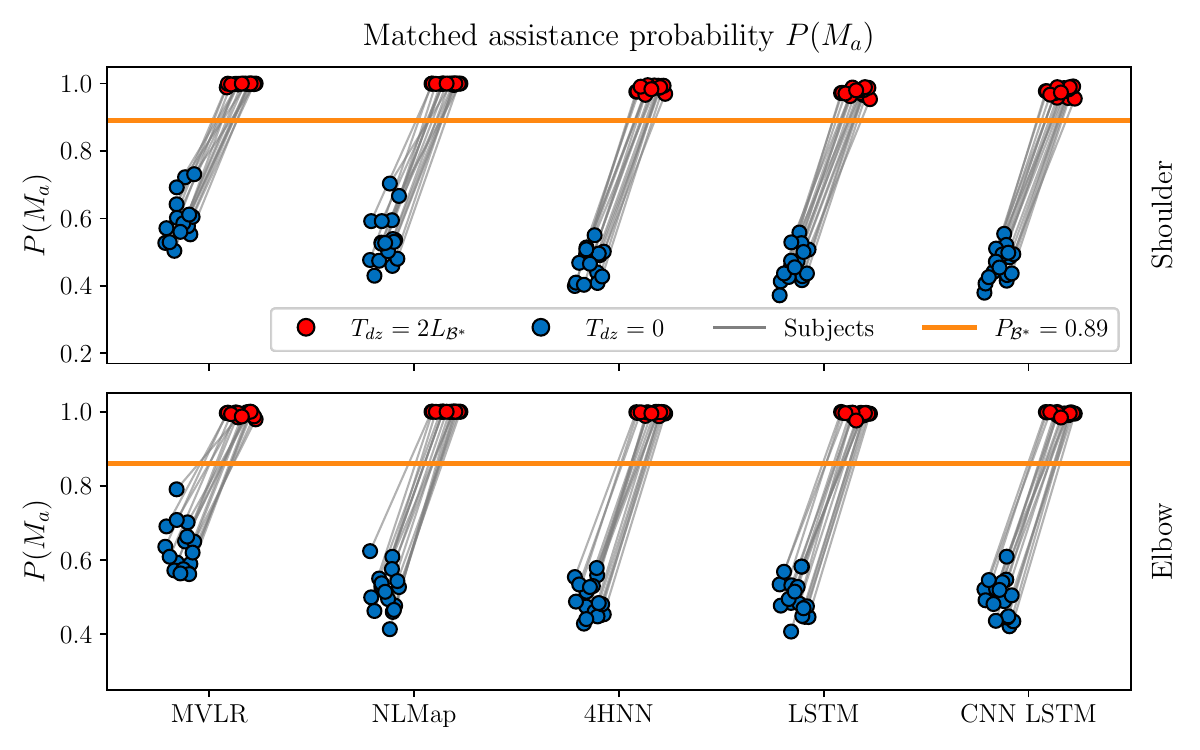}

\caption{Results on generalization performance guarantee across 17 subjects for the shoulder and elbow torques. The gray lines denote evaluated values from the same subject. \textbf{(A)} Torque estimation error of considered estimation models with corresponding generalization error upper bounds $L_{\mathcal{B}^*}$. \textbf{(B)} Matched assistance probability $P(M_a)$ of considered estimation models without the dead-zone $(T_{dz} = 0)$, with the dead-zone $(T_{dz} = 2L_{\mathcal{B}^*})$; the horizontal lines indicate the lower guarantee $P_{\mathcal{B}^*} = 0.89$ given by \eqref{eq: Pb*}.}
\label{fig: guarantee}
\end{figure*} 

\section{Results}
\label{sec: 4}

\subsection{Results on General Performance Guarantee}

\subsubsection{Guarantee on generalization error}

Fig. \ref{fig: guarantee}A reports the test errors of the shoulder and elbow torques for the five considered models, along with the corresponding upper bounds on the generalization error $L_{\mathcal{B}^*}$ computed by~(\ref{eq: LB}). It can be observed that, for all considered models, the test error of both shoulder and elbow torque was consistently below its corresponding generalization error upper bound, in agreement with Proposition~\ref{theo: Rademacher}. 

\subsubsection{Guarantee on matched assistance probability} 
Fig.~\ref{fig: guarantee}B compares the matched assistance probability $P(M_a)$ obtained from the outputs of the five considered EMG-to-torque models before and after dead-zone modulation with $T_{dz} = 2L_{\mathcal{B}^*}$. Note that $L_{\mathcal{B}^*}$ is calculated separately for each model, each subject, and each joint, as in Fig. \ref{fig: guarantee}A. Across all five model architectures and 17 subjects for both shoulder and elbow torque, dead-zone modulation substantially improved the matched assistance probabilities compared with the no-dead-zone condition, with average improvements of 0.38, 0.47, 0.51, 0.50, and 0.51 for the MVLR, NLMap, 4HNN, LSTM, and CNN-LSTM models, respectively. 

\begin{figure}[htp]
\includegraphics[scale=0.42]{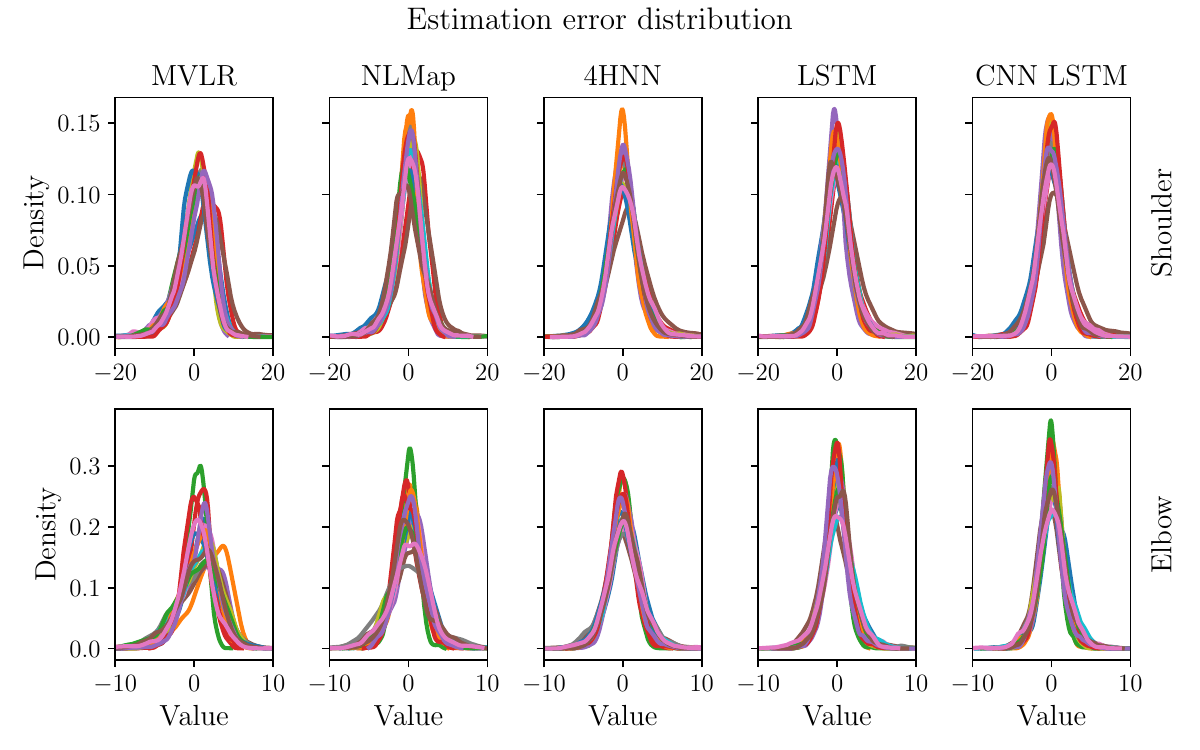}
\caption{The distribution of estimation error ($\tau_h - \hat{\tau}_h$)  obtained from 5 considered models across 17 subjects. Each curve corresponds to one subject.}
\label{fig: Error distribution}
\end{figure}

To assess the guarantee on the matched assistance probability, we first examine the validity of the zero-mean normality assumption for the torque estimation error $(\tau_h - \hat{\tau}_h)$. In Fig.~\ref{fig: Error distribution}, we illustrate the distributions of torque estimation errors obtained from the test sets of all 17 subjects for shoulder and elbow torque across the considered estimation models. The resulting error distributions exhibit a closely Gaussian shape with zero mean, notably for the 4HNN, LSTM, and CNN-LSTM models. This provides empirical evidence that supports the assumption in~\eqref{eq: assum} and, consequently, the practical relevance of the proposed theoretical analysis. After dead-zone modulation with the threshold $T_{dz} = 2L_{\mathcal{B}^*}$, it can be observed from Fig.~\ref{fig: guarantee}B that $P(M_a)$ remains above both theoretical lower bounds $P_{\mathcal{B}^*} = 0.89$ for all considered models and subjects, in accordance with Theorem \ref{theo: dz} and Remarks \ref{re: set dz}. Moreover, although the proposed design theoretically guarantees $P(M_a)\geq0.89$, the experimental results demonstrate a substantially higher matched assistance probability. Specifically, across the 17 subjects, two joint torques, and all considered models, the minimum observed $P(M_a)$ was $0.97$, indicating that the theoretical bound is conservatively satisfied in practice.

\subsubsection{Smoothness versus assistance}  
For each torque estimation model, the dead-zone threshold was individually determined as $T_{dz}=2L_{\mathcal{B}^*}$, using the corresponding error bound $L_{\mathcal{B}^*}$. This model-specific thresholding allows the five considered models to achieve a comparable matched assistance probability, as illustrated in Fig. \ref{fig: guarantee}B. Accordingly, the PCI given by \eqref{eq: PCI} is used as a secondary metric to assess the magnitude of the robot's positive contribution. 
\begin{figure}[ht!]
\includegraphics[scale=0.42]{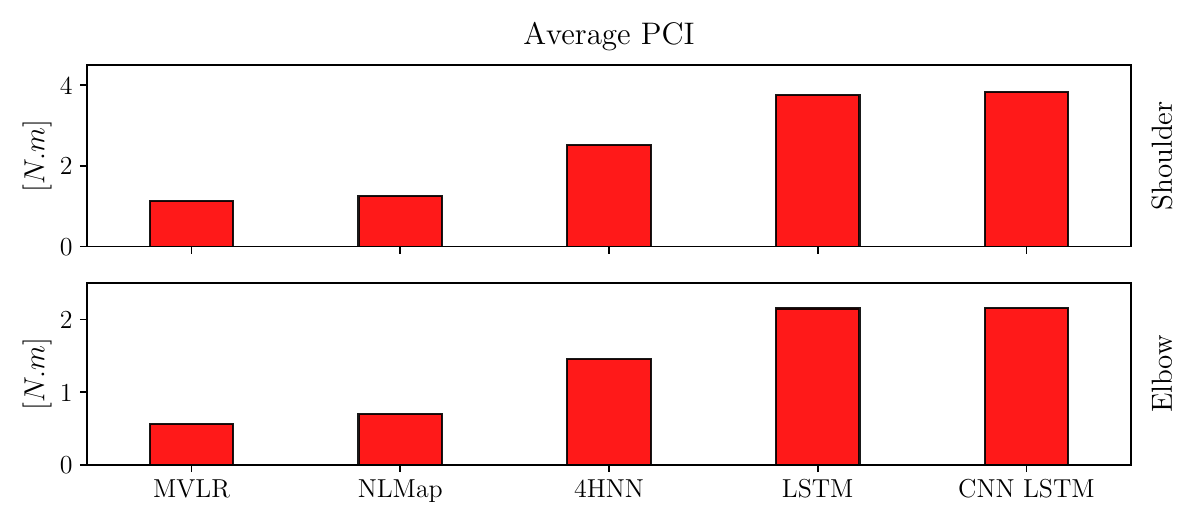}
\caption{Average PCI across 17 subjects and two joint torques derived by five considered torque estimation models: MVLR, NLMap, 4HNN, LSTM, and CNN-LSTM.}
\label{fig: PCI compare}
\end{figure}

Fig. \ref{fig: PCI compare} presents the average PCI obtained on the test set of 17 subjects under the five considered estimation models. With the dead-zone modulation, all models show a positive contribution. The level of positive contribution increases from MVLR and NLMap to the 4HNN, LSTM, and CNN-LSTM models. This indicates that higher-accuracy estimation models require a lower dead-zone threshold to guarantee matched assistance, thereby achieving a higher positive contribution.

\subsection{Results on EMG Assistance}
\label{sec: results assistance}

We assessed assistance performance against the baseline transparent condition TR, in which the robot typically follows the user without assisting. Consequently, the user retains full control of the motion, and natural trajectories are expected under this condition. However, the absence of robotic assistance in TR mode requires the user to generate their own effort to complete the task. Accordingly, this condition results in relatively high muscle activation \cite{Quesada2025less, hoang2025emg}. Therefore, an effective assistive controller should provide a level of movement smoothness comparable to that of the TR condition while reducing the user's muscular effort.

\begin{figure*}[htp]

\centering
\makebox[0.48\textwidth][l]{\fontsize{10pt}{10pt}\selectfont \textbf{A}}
\makebox[0.48\textwidth][l]{\fontsize{10pt}{10pt}\selectfont \textbf{B}}
\includegraphics[scale=0.42]{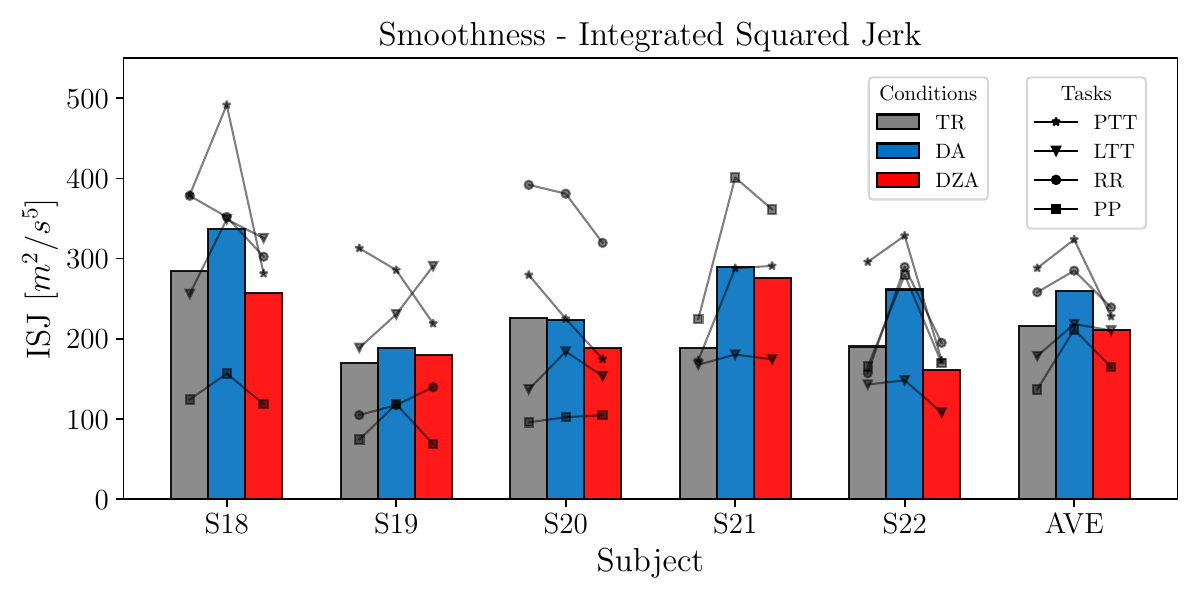}\hspace{0.5cm}\includegraphics[scale=0.42]{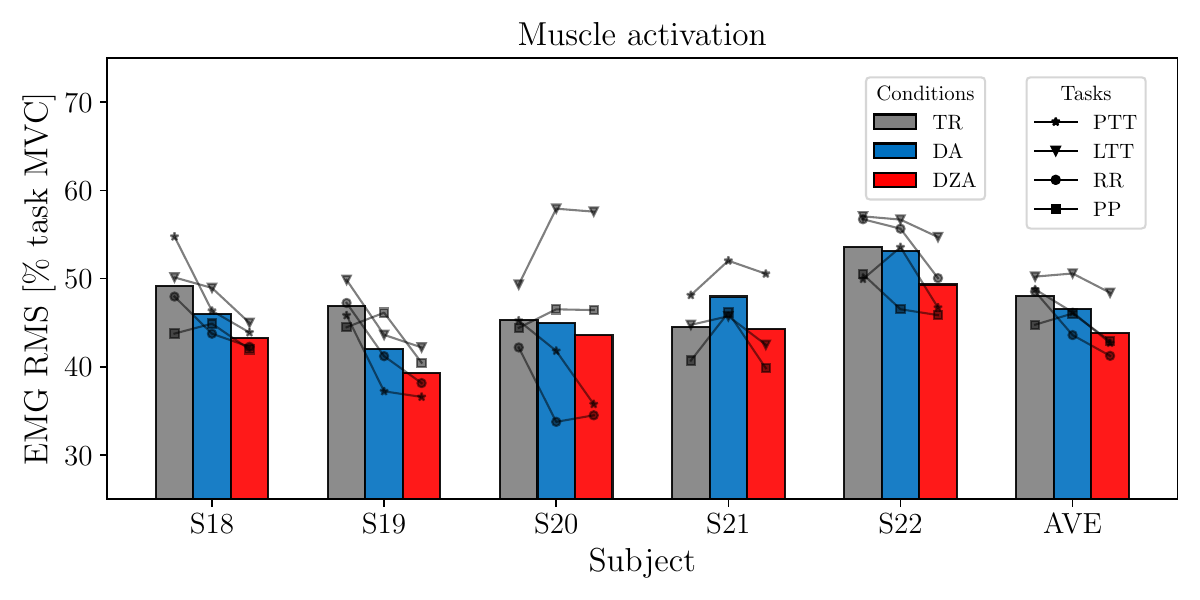}

\caption{Results on \textbf{(A)} motion smoothness and \textbf{(B)} muscle activation of 5 subjects across all testing tasks: PTT = Pure trajectory tracking, LTT = Load trajectory tracking, RR = Reach and return, PP = Pick and place. The last column, AVE, shows the average of the 5 subjects.}
\label{fig: assistance}
\end{figure*}

Fig.~\ref{fig: assistance}A illustrates the movement smoothness obtained by the three considered control conditions across all testing tasks. Since smoothness was quantified using the integrated squared jerk ISJ, lower values indicate smoother trajectories. On average, the proposed DZA achieves a smoothness level comparable to that of the TR mode, suggesting that the proposed strategy preserves human control of the task. In contrast, the DA controller yields the highest ISJ in four of the five testing subjects (i.e., all except S20) and the highest average value across subjects, indicating less smooth movements compared with both the TR mode and the proposed DZA. 

In terms of muscle activation, Fig.~\ref{fig: assistance}B shows that, averaged across all testing tasks, the two assistive conditions DA and DZA resulted in lower muscle activation than TR for four of the five testing subjects (i.e., all except S21). Compared with DA, the proposed DZA consistently achieved lower muscle activation across all five subjects. These results suggest that, although DA can generally reduce muscle activation relative to TR, its potential assistance mismatch may cause the robot to provide excessive or reversed torque relative to the user's intent, leading the user to generate counteracting torque to maintain movement stability, thereby diminishing the amount of muscle activation reduction. On the other hand, by limiting the effect of torque estimation errors, the proposed DZA aims to reduce this mismatch while maintaining effective assistance. Overall, when averaged across subjects and tasks, the DZA reduced muscle activation by $8.9\%$ relative to the TR mode, and $6.0\%$ relative to the DA condition. 

\subsection{Discussion}

This study characterizes the general control performance of the exoskeleton by a formal definition of matched assistance probability. Using a dead-zone mechanism, we provide a theoretical approach to set the dead-zone threshold that guarantees a lower bound on the matched assistance probability, which is verified in Fig. \ref{fig: guarantee}B. This improves the reliability of our method in assistive exoskeleton control by ensuring system performance even for previously unexamined tasks.

From a control perspective, the proposed strategy prioritizes human controllability and movement smoothness over robot assistance. In other words, the objective is not to maximize the amount of assistance, but to provide assistance that remains consistent with the user's intended movement and preserves natural task execution. For this purpose, in low-torque scenarios, the robot behaves similarly to TR mode, allowing the human to maintain movement smoothness. In contrast, reliable assistance is provided in high-torque scenarios to reduce human effort. This behavior was reflected in the experimental results, where the proposed EMG-assisted controller achieved smoother trajectories than DA while deriving lower muscle activation than both TR and DA. 


The design of our dead-zone threshold is related to the general performance of the EMG-to-torque model, as shown in Theorem \ref{theo: dz}. A small generalization error upper bound $L_{\mathcal{B}^*}$ enables achieving a high matched assistance probability with a small dead-zone threshold $T_{dz}$, allowing the robot to contribute more in assisting the task (see Fig. \ref{fig: guarantee} and \ref{fig: PCI compare}). These results demonstrate that the proposed framework can provide matched assistance even when the high-level EMG-to-torque estimation model varies in prediction quality. Nevertheless, the overall assistance performance may benefit from more accurate torque estimation. For the 4HNN model, Fig. \ref{fig: PCI compare} shows an average positive contribution of 2.5 N$\cdot$m and 1.5 N$\cdot$m for the shoulder and elbow torques, respectively. This level of robotic contribution was associated with an $8.9\%$ reduction in muscle activation compared with the TR mode. The LSTM and CNN-LSTM models yielded higher average positive contributions of 3.8 N$\cdot$m and 2.1 N$\cdot$m for the shoulder and elbow, respectively, suggesting the potential for a greater reduction in muscle activation than that achieved with 4HNN. While deep learning-based torque estimation models were not included in our practical robotic experiments due to their time-consuming training process, evaluating the proposed framework with such models in real-world human–robot experiments is still an important direction for future work.

In comparison with the scaled assistance (SCA) of \cite{Molinaro2024TaskAgnostic}, an offline experiment was conducted using the same 17-subject dataset \cite{quesada_dataset}. In Fig. \ref{fig: SCA vs DZA}, we compare SCA and the proposed DZA across 17 subjects in terms of the matched assistance probability $P(M_a)$ and the average PCI. For both shoulder and elbow torque, DZA achieves higher $P(M_a)$ and average PCI than SCA, suggesting potential advantages could be achieved with DZA over SCA. Future work will further establish a physical human–robot experiment between DZA and SCA to validate this statement. 
\begin{figure}[htp]
\includegraphics[scale=0.42]{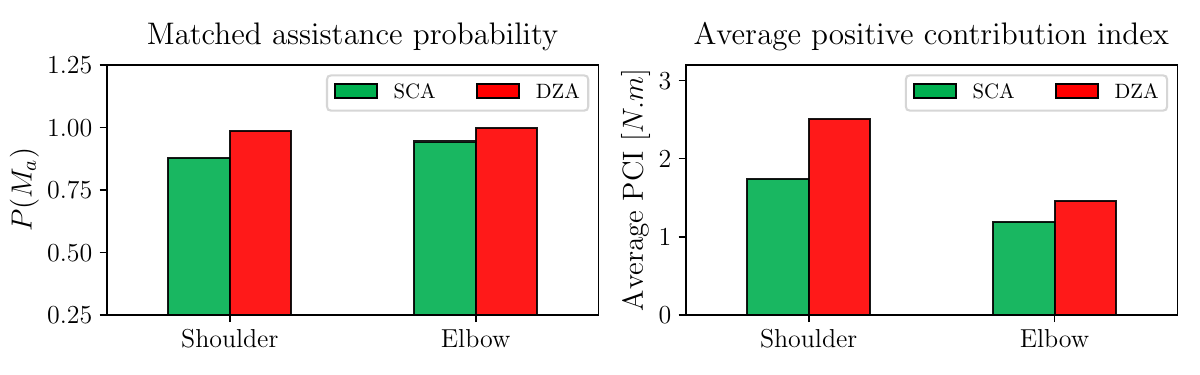}
\caption{Comparison between SCA and DZA across 17 subjects in terms of $P(M_a)$ and average PCI.}
\label{fig: SCA vs DZA}
\end{figure}

\section{Conclusion}
\label{sec: 5}
In this paper, we presented a theoretical framework for designing the robot’s desired interaction torque that guarantees a high probability of matched assistance. By integrating a dead-zone mechanism with a bounded-generalization-error EMG-to-torque model, our approach provides provable reliability in exoskeleton assistance with a formal guarantee on control performance. Unlike previous studies, our dead-zone threshold is derived theoretically from a general performance evaluation based solely on the training task, without conducting additional testing tasks on the robot. Practical experiments on the ABLE exoskeleton across multiple testing tasks verified the general effectiveness of the proposed approach, demonstrating guaranteed movement smoothness while reducing human physical effort. Future work will investigate the integration of the proposed method with deep learning-based torque estimation models (LSTM \cite{zhang2022lower} and CNN-LSTM \cite{hoang2025from}) in physical human–robot experiments, as well as compare the proposed method with other control strategies, such as the SCA \cite{Molinaro2024TaskAgnostic}, to further validate its effectiveness.


\section*{Appendix A: Proof of Proposition \ref{theo: Rademacher}}
\label{appendix 1}

At step $k^*$ of GD, we define the hypothesis class $\mathcal{F}_{k^*}$ containing $f_h(\cdot)$:
\begin{equation}
    {\mathcal{F}_{k^*}}=\left\{\boldsymbol{h}\mapsto \boldsymbol{b}^\top\boldsymbol{h}|\boldsymbol{b}\in \mathbb{R}^m,\|\boldsymbol{b}\|\leq\|\boldsymbol{a}(k^*)\|\right\}
\end{equation}
From Theorem 5.5 of \cite{ma22}, we have the Rademacher complexity ${\mathcal{R}}_{S}({\mathcal{F}_{k^*}})$ of $\mathcal{F}_{k^*}$ satisfies:
\begin{equation}
\label{eq: Rademacher}
    {\mathcal{R}}_{S}({\mathcal{F}_{k^*}}) \leq \frac{\|\boldsymbol{a}(k^*)\|}{\sqrt{n}}\mathcal{C}.
\end{equation}
Subsequently, according to Proposition 1 of \cite{XavierCF25} for the case $p=1$, with a probability at least $1-\delta$ over the sample $\mathcal{S}$ of size $n$, we have:
\begin{equation}
\label{eq: bound}
    L_{\mathcal{D}}(f_h)\leq\frac{1}{n}\sum_{i=1}^n|v_{i}(k^*)|+ 2{\mathcal{R}}_{S}({\mathcal{F}_{k^*}}) + 3R\sqrt{\frac{\log \frac{2}{\delta}}{2n}}.
\end{equation}
Substituting (\ref{eq: Rademacher}) into (\ref{eq: bound}), we obtain \eqref{eq: lb*} and this completes the proof.

\section*{Appendix B: Extended Version of Theorem \ref{theo: dz}}
\begin{theorem}\label{theo: dz 2} If the estimation error follows a normal distribution $\mathcal{N}(\mu,\sigma^2)$ and generalization error $L_{\mathcal{D}} = \mathbb{E}\left(|{\tau}_h - \hat{\tau}_h|\right)$ satisfies $L_{\mathcal{D}}\leq L_{\mathcal{B}^*}$, then, by using the dead-zone mechanism with the threshold  $T_{dz} \geq L_{\mathcal{B}^*}$, the matched assistance probability $P(M_a)$ is guaranteed to satisfy:
\begin{equation}
\label{eq: worst}
P\left(M_a\right) 
\geq  \Phi\left(\sqrt{\frac{2}{\pi}}\frac{T_{dz} - L_{\mathcal{B}^*}}{L_{\mathcal{B}^*}}\right) + \Phi\left(\sqrt{\frac{2}{\pi}}\frac{T_{dz} + L_{\mathcal{B}^*}}{L_{\mathcal{B}^*}}\right) - 1
\end{equation}
\end{theorem}
\begin{proof}

From Proposition \ref{prop: relation PMa} and $\left( {\tau}_h - \hat{\tau}_h\right)\sim\mathcal{N}(\mu,\sigma^2)$:
\begin{equation}
\label{eq: PMa by g}
\begin{aligned}
    P(M_a) =&\;P\left(\tau_h - \hat{\tau}_h \leq T_{dz}\right) - P\left(\tau_h - \hat{\tau}_h \geq -T_{dz}\right)\\
    =&\;\Phi\left(\frac{T_{dz} - |\mu|}{\sigma}\right) + \Phi\left(\frac{T_{dz} + |\mu|}{\sigma}\right) - 1.
\end{aligned}
\end{equation}

We next introduce the auxiliary function
\begin{equation}
\label{eq: g def}
\mathcal{G}(u,v)=
\Phi\!\left(\frac{T_{dz}-u}{v}\right)
+\Phi\!\left(\frac{T_{dz}+u}{v}\right)-1,
\end{equation}
where \(u\in[0,\beta]\) and \(v\in(0,\gamma]\) and prove that for $T_{dz}\geq \beta$: 
\begin{equation}
\label{eq: bound g uv}
    \mathcal{G}(u,v) \geq \mathcal{G}(\beta,\gamma) \; \forall\;(u,v)\in[0,\beta]\times(0,\gamma].
\end{equation} 
For any fixed $v \in (0,\gamma]$, taking the derivative of $\mathcal{G}(u,v)$ with respect to $u$ yields:
\begin{equation}
    \frac{\partial \mathcal{G}}{\partial u} = \frac{1}{v}\left[\phi\left(\frac{T_{dz} + u}{v}\right) - \phi\left(\frac{T_{dz} - u}{v}\right)\right]
\end{equation}
where $\phi(\cdot)$ is the probability density function (PDF) of the standard normal distribution. 

Since $T_{dz} \geq \beta$, we have $T_{dz} + u \geq T_{dz} - u \geq 0$ for all $u \in [0,\beta]$; therefore $ \phi\left(\frac{T_{dz} + u}{v}\right) \leq \phi\left(\frac{T_{dz} - u}{v}\right)$ and: 
\begin{equation}
\label{eq: g' u}
    \frac{\partial \mathcal{G}}{\partial u}\leq 0 \; \forall\;(u,v)\in[0,\beta]\times(0,\gamma].
\end{equation}
For any fixed $u \in [0,\beta]$, taking the derivative of $\mathcal{G}(u,v)$ with respect to $v$ yields:
$$\frac{\partial \mathcal{G}}{\partial v} =- \frac{T_{dz}+u}{v^2} \phi\left(\frac{T_{dz} + u}{v}\right)- \frac{T_{dz} - u}{v^2}\phi\left(\frac{T_{dz} - u}{v}\right).$$
Since $T_{dz} + u \geq T_{dz} - u \geq 0$ and $\phi(\cdot)\geq 0$, we have:
\begin{equation}
\label{eq: g' v}
    \frac{\partial \mathcal{G}}{\partial v}\leq 0 \; \forall\;(u,v)\in[0,\beta]\times(0,\gamma].
\end{equation}
From (\ref{eq: g' u}) and (\ref{eq: g' v}), we derive \eqref{eq: bound g uv}.

We now prove that $|\mu|\leq L_{\mathcal{B}^*}\;\text{and}\; \sigma \leq \sqrt{{\pi}/{2}}L_{\mathcal{B}^*}$. From the normal distribution $\mathcal{N}\left(\mu,\sigma^2\right)$, we have:
\begin{equation}
\label{eq: mu proof}
    |\mu| = |\mathbb{E}\left({\tau}_h - \hat{\tau}_h\right)|\leq\mathbb{E}\left(|{\tau}_h - \hat{\tau}_h|\right)\leq L_{\mathcal{B}^*}
\end{equation}
We also have:
\begin{equation}
\label{eq: E def}
    \mathbb{E}\left(|{\tau}_h - \hat{\tau}_h|\right) = 2\sigma\phi\left(\frac{|\mu|}{\sigma}\right) + |\mu|\left(2\Phi\left(\frac{|\mu|}{\sigma}\right) - 1\right).
\end{equation}
Let $g(s)$ a function of non-negative variable $s$ defined by:
\begin{equation}
    g(s)=2\phi\left(s\right) + s\left(2\Phi\left(s\right) - 1\right).
\end{equation}
By differentiating $g(s)$ with respect to $s$, we obtain:
\begin{equation}
\label{eq: gs'}
    g'(s) = 2\phi'\left(s\right) + 2\Phi\left(s\right) - 1 + 2s\Phi'\left(s\right).
\end{equation}
Using the property of the PDF function, $\phi'\left(s\right) = -s\phi(s)$, and $\Phi'\left(s\right) = \phi\left(s\right)$, we derive $g'(s) =2\Phi\left(s\right) - 1 \geq 0$ for all $s\geq 0$. Accordingly, $g(s) \geq g(0)\;\forall\;s\geq 0$. Combine with (\ref{eq: E def}) yields: 
\begin{equation}
    \mathbb{E}\left(|{\tau}_h - \hat{\tau}_h|\right) = \sigma g\left(\frac{|\mu|}{\sigma}\right)\geq\sigma g(0) = \sigma \sqrt{\frac{2}{\pi}}.
\end{equation}
Thus, we obtain:
\begin{equation}
\label{eq: sigma proof}
    \sigma \leq \sqrt{\frac{\pi}{2}}\mathbb{E}\left(|{\tau}_h - \hat{\tau}_h|\right) \leq \sqrt{\frac{\pi}{2}}L_{\mathcal{B}^*}.
\end{equation}

Finally, from \eqref{eq: PMa by g}, \eqref{eq: bound g uv}, \eqref{eq: mu proof} and \eqref{eq: sigma proof}, replacing $\beta$  and $\gamma$ in \eqref{eq: bound g uv} by $L_{\mathcal{B}^*}$ and $\sqrt{\pi/2}L_{\mathcal{B}^*}$, respectively, we have:
\begin{equation*}
\begin{aligned}
    P(M_a) =&\;\mathcal{G}(|\mu|, \sigma)    \geq \mathcal{G}(L_{\mathcal{B}^*}, \sqrt{\pi/2}L_{\mathcal{B}^*})\\
    =&\;\Phi\left(\sqrt{\frac{2}{\pi}}\frac{T_{dz} - L_{\mathcal{B}^*}}{L_{\mathcal{B}^*}}\right) + \Phi\left(\sqrt{\frac{2}{\pi}}\frac{T_{dz} + L_{\mathcal{B}^*}}{L_{\mathcal{B}^*}}\right) - 1
\end{aligned}
\end{equation*}
i.e. \eqref{eq: worst}.
\end{proof}

\bibliographystyle{IEEEtran}
\bibliography{bibliobis}

\vfill

\end{document}